\documentclass[]{template}

\usepackage[utf8]{inputenc}             
\usepackage[T1]{fontenc}                
\usepackage{colortbl}
\usepackage{multicol}
\usepackage[dvipsnames]{xcolor}         

\usepackage{latexsym}

\usepackage{graphicx}
\usepackage{float}
\usepackage{subcaption}
\usepackage{wrapfig}
\usepackage{lipsum}
\usepackage[utf8]{inputenc}
\usepackage[T1]{fontenc}
\usepackage{hyperref}
\usepackage{url}
\usepackage{booktabs}
\usepackage{amsfonts}
\usepackage{amsmath}
\usepackage{amssymb}
\usepackage{nicefrac}
\usepackage{microtype}
\usepackage{xcolor}
\usepackage{algorithmic}
\usepackage{graphicx}
\usepackage{subcaption}
\usepackage{multirow}
\usepackage{makecell}
\usepackage{enumitem}
\usepackage{wrapfig}
\usepackage{comment}
\usepackage[table]{xcolor}
\usepackage{amsmath,amssymb,amsthm}
\newtheorem{theorem}{Theorem}[section]

\newtheorem{proposition}[theorem]{Proposition}

\newcommand{\ie}{\textit{i.e.}}
\newcommand{\eg}{\textit{e.g.}}

\usepackage{bm}

\usepackage{tabularx} 
\usepackage{ragged2e} 
\newcolumntype{L}{>{\RaggedRight\hangafter=1\hangindent=0em}X}

\usepackage{enumitem}

\usepackage{mathtools}

\setboolean{logo}{true}    

\usepackage[linesnumbered,ruled,vlined]{algorithm2e}

\hypersetup{
    colorlinks=true,
    linkcolor=red,
    citecolor=Cerulean,
    filecolor=magenta,      
    urlcolor=magenta,
}

\usepackage[capitalize,noabbrev]{cleveref}
\crefname{section}{§}{§§}
\Crefname{section}{§}{§§}

\usepackage{calligra}
\DeclareMathAlphabet{\mathcalligra}{T1}{calligra}{m}{n}

\usepackage{pifont}

\theoremstyle{plain}
\theoremstyle{definition}
\theoremstyle{remark}

\renewcommand{\paragraph}[1]{\vspace{1mm}\noindent\textbf{#1}}

\DeclareCaptionLabelFormat{cont}{#1~#2\alph{ContinuedFloat}}
\usepackage[most]{tcolorbox}
\tcbset{
  promptbox/.style={
    top=10pt,
    colback=lightgray!20,
    colframe=Black,
    colbacktitle=NavyBlue,
    enhanced,
    center,
    attach boxed title to top center={yshift=-0.1in,xshift=0.0in},
    boxed title style={boxrule=0pt,colframe=white,},
  }
}
\newtcolorbox{promptbox}[2][]{promptbox, title=#2,#1}
\tcbset{
  takeawaybox/.style={
    top=10pt,
    colback=lightgray!20,
    colframe=Black,
    colbacktitle=BurntOrange,
    enhanced,
    center,
    attach boxed title to top center={yshift=-0.1in,xshift=0.0in},
    boxed title style={boxrule=0pt,colframe=white,},
  }
}
\newtcolorbox{takeawaybox}[2][]{takeawaybox, title=#2,#1}
\tcbset{
  observationbox/.style={
    top=10pt,
    colback=lightgray!20,
    colframe=Black,
    colbacktitle=YellowGreen,
    enhanced,
    center,
    attach boxed title to top center={yshift=-0.1in,xshift=0.0in},
    boxed title style={boxrule=0pt,colframe=white,},
  }
}
\newtcolorbox{observationbox}[2][]{observationbox, title=#2,#1}

\usepackage{xspace}

\newcommand\blfootnote[1]{%
  \begingroup
  \renewcommand\thefootnote{}\footnote{#1}%
  \addtocounter{footnote}{-1}%
  \endgroup
}

\usepackage{CJK}

\title{Group Entropy-Controlled Policy Optimization}

\author{Guangran Cheng}
\author{Chengqi Lyu}
\author{Songyang Gao}
\author{Wenwei Zhang$^\dagger$}
\author{Kai Chen$^\dagger$}

\affil{Shanghai AI Laboratory}

\begin{abstract}
Entropy control has become an effective tool in reinforcement learning (RL) of large language models (LLMs), helping balance exploration-exploitation trade-off during alignment process.
Such RL paradigm is often conducted on mixtures of heterogeneous tasks, which induce distinct entropy regimes under the same policy, making global or token-level entropy regulation insufficient to corresponding heterogeneous needs of exploration.
This heterogeneity further makes GRPO-style normalized advantages induce an entropy-dependent bias, making advantage signals across prompt groups statistically non-comparable.
To address this issue, we propose Group Entropy-Controlled Policy Optimization (GEPO), a lightweight extension to GRPO that uses group entropy, estimated from existing grouped samples to perform entropy-conditioned asymmetric advantage shaping.
GEPO attenuates positive advantages in low-entropy groups to reduce over-exploitation, and negative advantages in high-entropy groups to preserve exploration, with adaptive thresholds derived from historical entropy statistics.
Extensive experiments on two base models across thirteen benchmarks spanning mathematics, physics, science, code generation, and instruction following show that GEPO consistently outperforms GRPO and recent entropy-controlled methods, delivering balanced cross-task improvements while preserving task-specific exploration levels throughout training. 
\end{abstract}
\vspace{-2cm}

\begin{document}

\maketitle

\begin{center}
\includegraphics[width=1\textwidth]{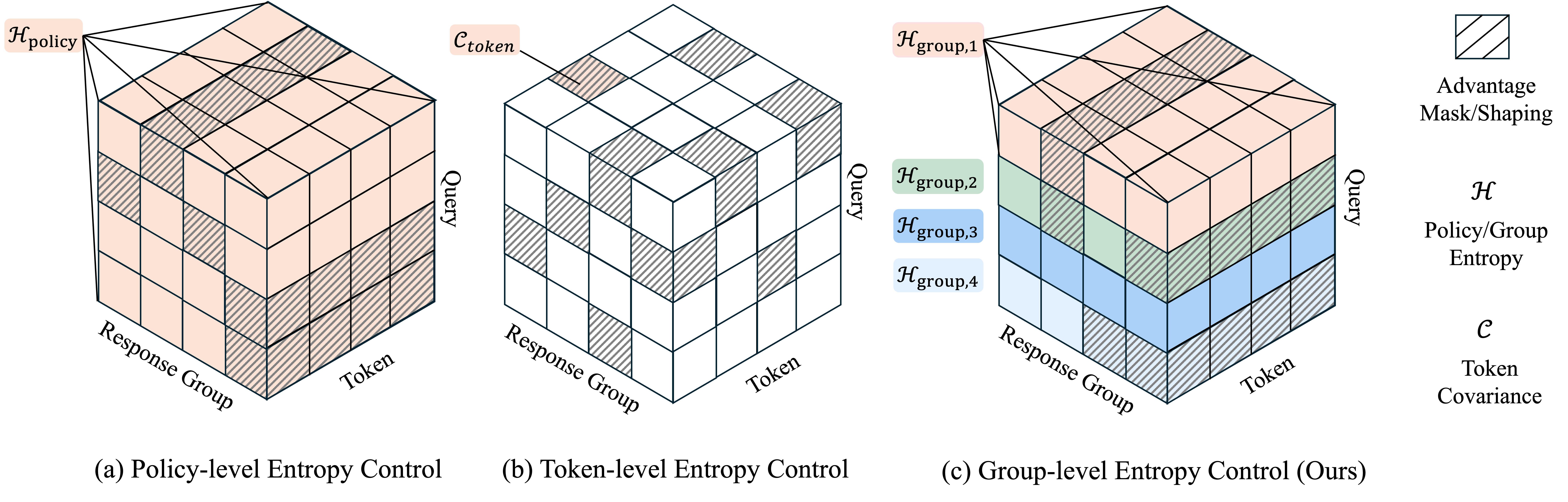}
\captionof{figure}{Comparison of entropy control methods at different granularities. 
(a): Policy-level methods calculate the entropy of policy over the current batch and unanimously apply it to regulate all responses; 
(b): Token-level methods calculate the covariance of tokens over the current batch and individually apply it to regulate each token;
(c): GEPO calculate group entropy over the responses of each prompt and respectively regulate responses, enabling task-specific exploration-exploitation trade-off across heterogeneous tasks.}
\label{fig:first_page}
\end{center}

\blfootnote{$\dagger$ Corresponding authors: Wenwei Zhang (zhangwenwei@pjlab.org.cn), Kai Chen (chenkai@pjlab.org.cn)}

\section{Introduction}

Reinforcement learning (RL) has recently emerged as an effective post-training paradigm for enhancing both alignment and reasoning capabilities of Large Language Models (LLMs), yielding substantial improvements across a wide range of downstream tasks \cite{shao2024deepseekmath,team2025kimi,jaech2024openai,hu2025open}.
Nevertheless, balancing exploration and exploitation remains a longstanding challenge in RL \cite{sutton1998reinforcement,bellemare2016unifying}, which becomes increasingly pronounced in LLM post-training due to the optimization over heterogeneous task mixtures.
Specifically, modern LLM post-training is typically performed on mixtures of diverse tasks, spanning instruction following, coding, writing, and mathematical reasoning, which differ substantially in structure, solution diversity, and uncertainty of policy exploration.
Therefore, a central challenge in applying RL to multi-task mixtures is exploration-exploitation trade-off across heterogeneous tasks.

Recent RL methods seek to balance exploration and exploitation through entropy-controlled mechanisms operating either at the policy level (Figure \ref{fig:first_page}(a)) \cite{yang2025entropic,shen2025qwenlong} or at the token level (Figure \ref{fig:first_page}(b)) \cite{cheng2026reasoning,cui2025entropy}.
However, these approaches primarily treat entropy as a global indicator for sustaining stable RL training, overlooking the optimization imbalance among heterogeneous tasks.
Specifically, Group Relative Policy Optimization (GRPO), one of the dominant algorithms in LLM RL, standardizes rewards within each prompt group to construct relative advantages, implicitly treating the resulting advantage signals as unbiased across groups.
Yet whether this assumption remains valid when prompt groups exhibit substantially different levels of entropy, and how such entropy heterogeneity affects joint multi-task optimization, remain largely under-explored in existing entropy-controlled methods.

In this work, we make the first attempt to delve deep into this problem, and reveal that the group entropy (\ie, the entropy calculated from the response groups of each task) indeed varies from their corresponding task characteristics, which may implicitly induces a structural bias in GRPO's normalized advantages, causing the optimization of different tasks imbalanced.
Based on this insight, we propose Group Entropy-Controlled Policy Optimization (GEPO) (Figure \ref{fig:first_page}(c)), a lightweight and model-agnostic extension to related group-based policy optimization methods that performs entropy-conditioned asymmetric advantage shaping at the group level.
The key idea is to use group entropy as a diagnostic signal to identify and mitigate the optimization bias induced by entropy heterogeneity.
Specifically, GEPO attenuates positive advantages in low-entropy groups to prevent over-exploitation that would further amplify the entropy gap, while attenuating negative advantages in high-entropy groups to avoid prematurely suppressing exploration.
This shaping is asymmetric because we empirically find that low-entropy groups are more susceptible to aggressive intervention, which may trigger length collapse, and therefore require milder attenuation than high-entropy groups.
Furthermore, the entropy boundaries for advantage shaping are adaptively derived from historical entropy statistics and smoothed with exponential moving average, enabling automatic calibration across base models and training stages without task-specific tuning.

Extensive experiments on Intern-S1-mini and Qwen3.5-9B across thirteen diverse benchmarks demonstrate that GEPO consistently improves both average performance and inter-task balance, achieving state-of-the-art results among GRPO and entropy-aware RL methods. 
Further analysis further shows that GEPO induces more stable optimization trajectories, and consistently benefits tasks with diverse exploration states. 
Notably, these gains are obtained with zero additional sampling cost and can be directly incorporated into existing group-based policy optimization methods.

\section{Method}
\subsection{Preliminary}
Group Relative Policy Optimization (GRPO) (\cite{shao2024deepseekmath}) is a reinforcement learning (RL) optimization framework for large language model (LLM) post-training that eliminates the need for a separate value function by estimating advantages from grouped samples. 
Given a prompt $x$, GRPO samples $K$ responses $\{y_1, y_2, \ldots, y_K\}$ from the current policy $\pi_\theta$, obtains rewards $\{r_1, r_2, \ldots, r_K\}$, and computes normalized advantages:
\begin{equation}
   A_i = \frac{r_i - \text{mean}(\{r_j\}_{j=1}^K)}{\text{std}(\{r_j\}_{j=1}^K)}, \quad i = 1, \ldots, K.
    \label{eq:grpo_advantage}
\end{equation}
The GRPO objective maximizes the clipped surrogate:
\begin{equation}
    \mathcal{L}_{\text{GRPO}}(\theta) = \mathbb{E}_{x, \{y_i\}} \left[ \frac{1}{K} \sum_{i=1}^{K} \frac{1}{T_i} \sum_{t=1}^{T_i} \left( \min\left( \rho_{i,t} A_i, \, \text{clip}(\rho_{i,t}, 1\!-\!\epsilon, 1\!+\!\epsilon) A_i \right)\right) \right],
    \label{eq:grpo_loss}
\end{equation}
where $\rho_{i,t} = \frac{\pi_\theta(y_{i,t} | x, y_{i,<t})}{\pi_{\theta_{\text{old}}}(y_{i,t} | x, y_{i,<t})}$ is the importance sampling ratio, and $T_i$ is the length of response $y_i$.

\subsection{Group Entropy and Optimization Dynamics}
\label{sec:group_entropy}

We define group entropy, a signal computed for each response group that reflects the current exploration state of the policy for each task, and examine its empirical properties under GRPO's grouped sampling procedure.

For a prompt $x$, we define the sequence-level entropy as
\begin{equation}
    H(\pi_\theta(\cdot \mid x))= \mathbb{E}_{y \sim \pi_\theta(\cdot \mid x)}\!\left[-\log \pi_\theta(y \mid x)\right] = \mathbb{E}_{y \sim \pi_\theta(\cdot \mid x)}\!\left[-\sum_{t=1}^{T} \log \pi_\theta(y_t \mid y_{<t}, x)\right],
\label{eq:sequence_entropy}
\end{equation}
which directly characterizes the uncertainty of the policy on prompt $x$: low values indicate that $\pi_\theta$ concentrates its probability on a narrow set of responses (potential over-exploitation), while high values indicate that $\pi_\theta$ spreads probability across diverse responses (active exploration, but potentially unstable optimization).

However, computing $H(\pi_\theta(\cdot|x))$ exactly is intractable for LLMs with the large vocabulary size and response length for reasoning tasks. 
The expectation form of Equation~\ref{eq:sequence_entropy} admits a natural Monte Carlo estimator. 
In GRPO, group responses $\{y_1, \ldots, y_K\}$ are sampled i.i.d.\ from $\pi_\theta(\cdot|x)$ for each prompt $x$. 
Therefore, we can define \textbf{group entropy} as an estimator of the group-based sequence-level entropy:
\begin{equation}
    \hat{H}_{\text{g}}(x) \;=\; -\frac{1}{K} \sum_{i=1}^{K} \sum_{t=1}^{T_i} \log \pi_\theta(y_{i,t} \mid y_{i,<t}, x).
    \label{eq:group_entropy_mc}
\end{equation}
As a standard Monte Carlo estimator,
$\hat{H}_{\mathrm g}(x)$ is an unbiased estimate of
$H(\pi_\theta(\cdot\mid x))$, with variance decreasing as
$O(K^{-1})$.
In practice, we normalize by the average response length to obtain the per-token group entropy $\mathcal{H}_{\text{g}}(x) = \hat{H}_{\text{g}}(x) / \bar{T}$, where $\bar{T} = \frac{1}{K}\sum_{i=1}^K T_i$, to improve comparability across groups with different response lengths.

We now present two empirical observations from the first rollout of GRPO training. 
Specifically, group entropy is computed for each response group associated with a prompt, while semantic domains are used only to aggregate and visualize these group-level quantities. 
These observations motivate the method design in Section~\ref{sec:gepo_method}.
Detailed experimental settings are in Section~\ref{experimental_setting}.

\noindent\textbf{Observation 1: group entropy is heterogeneous across domains.}
Figure~\ref{fig:initial_entropy} shows the domain-conditional distributions of group entropy.
Under the same base policy, prompts from different domains occupy markedly different entropy regimes: for example, physics prompts concentrate around $\mathcal{H}_{\text{g}} \approx 0.49$, whereas general knowledge prompts are centered near $\mathcal{H}_{\text{g}} \approx 0.81$.
Consequently, a single global entropy statistic cannot adequately characterize the exploration states of all prompts in a multi-domain training mixture.

\begin{figure}[t]
  \centering
  \begin{subfigure}[b]{0.48\textwidth}
    \centering
    \includegraphics[width=\textwidth]{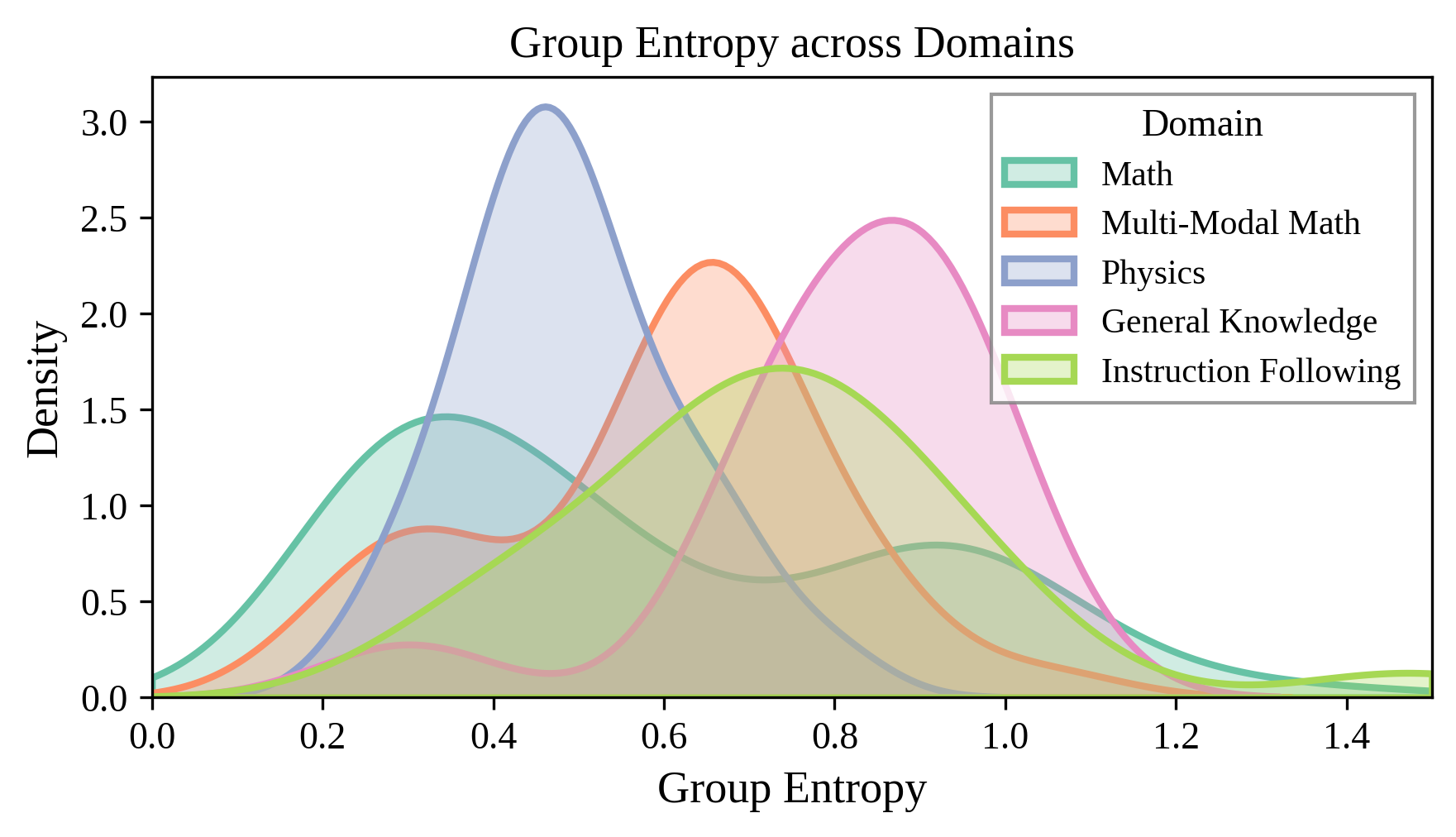}
    \caption{group entropy by domain.}
    \label{fig:initial_entropy}
  \end{subfigure}
  \hfill
  \begin{subfigure}[b]{0.48\textwidth}
    \centering
    \includegraphics[width=\textwidth]{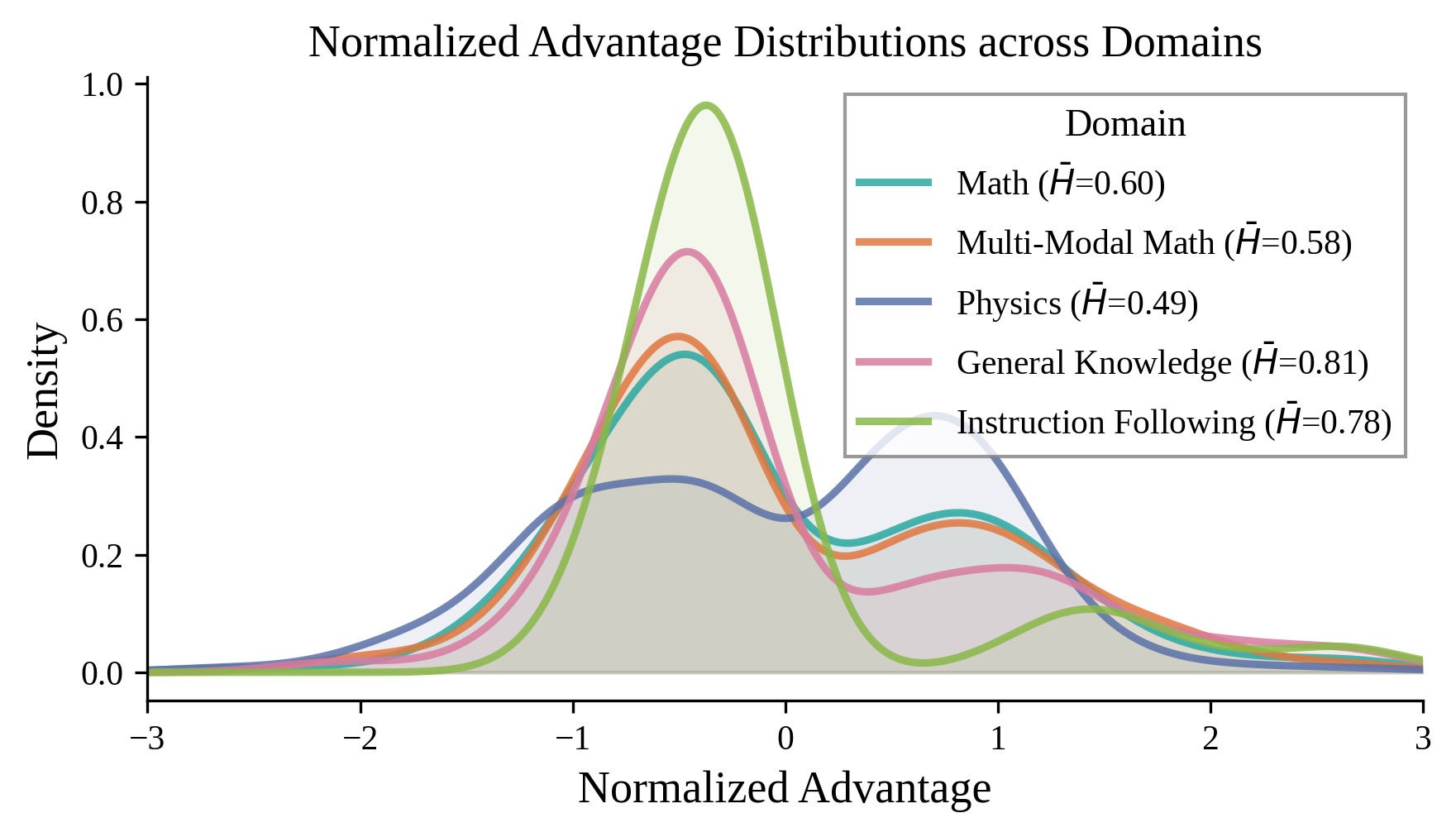}
    \caption{Normalized advantage distributions by domain.}
    \label{fig:advantage_by_task}
  \end{subfigure}
  \caption{Preliminary observations from the first GRPO rollout before any policy update ($K\!=\!16$ responses per prompt).
  \textbf{(a)}~Domain-conditional distributions of group entropy.
  \textbf{(b)}~Domain-conditional distributions of GRPO-normalized advantages.
  Together, the panels show that semantic domains occupy different exploration and credit-assignment regimes within the same training mixture.}
  \label{fig:preliminary_obs}
\end{figure}

\noindent\textbf{Observation 2: Domain-level entropy regimes coincide with distinct advantage profiles.}
Figure~\ref{fig:advantage_by_task} shows that domains occupying different group entropy regimes also exhibit markedly different normalized-advantage distributions:
\begin{itemize}[leftmargin=*]
\item The low-entropy Physics domain ($\bar{\mathcal{H}}_g \!=\! 0.49$, accuracy 48\%) produces a nearly \emph{symmetric bimodal} advantage distribution (skewness $\approx 0$), with balanced positive and negative signals.
\item The medium-entropy Math domain ($\bar{\mathcal{H}}_g \!=\! 0.60$, accuracy 39\%) exhibits moderate right skew ($\approx 0.75$), with more negative than positive advantages.
\item The high-entropy Instruction Following domain ($\bar{\mathcal{H}}_g \!=\! 0.78$, accuracy 16\%) yields extreme right skew (skewness $= 2.32$, kurtosis $= 4.84$): the vast majority of responses cluster near a small negative advantage, while the rare correct responses become extreme positive outliers.
\end{itemize}

The observed asymmetry is not caused by entropy alone, but arises mechanically from the lower group accuracy that accompanies the high-entropy regime in this rollout. 
When group accuracy $p\to 0$, the zero-mean constraint forces positive advantages to concentrate on rare correct responses with magnitude $|A_+|=\sqrt{(1-p)/p}\gg 1$, while negative advantages are diluted across the majority with $|A_-|=\sqrt{p/(1-p)}\ll 1$ (see Appendix~\ref{app:advantage_asymmetry_analysis} for the full derivation).
Beyond this intra-group asymmetry, we analyze a more fundamental \emph{inter-group} issue: standard group normalization does not guarantee comparable advantage signals across groups with different entropy levels.

To formalize this, let $\nu_x$ denote a reference measure over the response space $\mathcal{A}_x$ (\eg, the uniform distribution over feasible responses), and define the \emph{reference reward cumulative distribution function (CDF)} $F_x^\nu(t):=\nu_x(\{a:\mu_x(a)\le t\})$ and the \emph{policy-weighted reward CDF} $F_x^\pi(t):=\pi_x(\{a:\mu_x(a)\le t\})$, where $\mu_x(a)=\mathbb{E}[R\mid x,a]$.

\begin{proposition}[Entropy-Dependent Distribution Distortion]
\label{prop:distortion}
For any prompt $x$ with conditional policy entropy $H(x)=-\sum_a \pi_x(a)\log\pi_x(a)$ and response space of size $N_x = |\mathcal{A}_x|$, the Kolmogorov distance between the policy-weighted and reference reward distributions satisfies
\begin{equation}
    \sup_t \left|F_x^\pi(t) - F_x^\nu(t)\right| \;\leq\; \mathrm{TV}(\pi_x, \nu_x) \;\leq\; \sqrt{\frac{1}{2}\left(\log N_x - H(x)\right)},
    \label{eq:distortion_bound}
\end{equation}
where the second inequality holds when $\nu_x$ is the uniform distribution over $\mathcal{A}_x$, following from Pinsker's inequality (\cite{tsybakov2009introduction}) applied to $\mathrm{KL}(\pi_x\|\nu_x)=\log N_x - H(x)$.
\end{proposition}
Proposition~\ref{prop:distortion} shows that low policy entropy permits a larger worst-case deviation between the policy-weighted and reference reward distributions. Thus, low-entropy prompts can be more sensitive to policy-induced reweighting of the response space, although the bound does not imply that the maximal deviation is attained for every prompt.

\begin{proposition}[Structural Bias of Normalized Advantages]
\label{prop:adv_bias}
\leavevmode\indent
Let $Z_x^\pi(a) = (\mu_x(a)-m_x^\pi)/\sigma_x^\pi$ denote the advantage under policy-weighted moments, and $Z_x^\nu(a) = (\mu_x(a)-m_x^\nu)/\sigma_x^\nu$ the advantage under reference moments. Assume $|\mu_x(a)|\le M$ and $\sigma_x^\pi, \sigma_x^\nu \ge \sigma_0 > 0$ for all $x,a$. Then there exists a constant $C=C(M,\sigma_0)>0$ such that
\begin{equation}
    \sup_{a\in\mathcal{A}_x} \left|Z_x^\pi(a) - Z_x^\nu(a)\right| \;\leq\; C\,\mathrm{TV}(\pi_x,\nu_x) \;\leq\; C\sqrt{\frac{1}{2}\left(\log N_x - H(x)\right)}.
    \label{eq:adv_bias_bound}
\end{equation}
\end{proposition}
Proposition~\ref{prop:adv_bias} further shows that per-group centering and scaling do not guarantee a common statistical reference frame across prompts: the potential discrepancy between policy-weighted and reference-normalized advantages remains entropy dependent.

\noindent\textbf{Optimization consequences.}\hspace{0.3cm}
The intra-group skewness (Observation~2) and inter-group structural bias (Propositions~\ref{prop:distortion}--\ref{prop:adv_bias}) jointly create entropy-dependent optimization pressures that manifest at two levels:

Within each high-entropy group, the advantage asymmetry produces two concurrent pathologies.
First, \emph{diffuse and noisy suppression}: incorrect responses receive small negative advantages spread  over many diverse reasoning paths, making the suppression signal unreliable and prone to prematurely  penalizing exploratory trajectories that may still contain useful partial reasoning.
Second, \emph{inconsistent reinforcement}: rare correct responses receive large positive advantages, but  different rollouts often surface different reasoning paths, so the reinforcement signal does not accumulate coherently over training steps.

\emph{Across} groups, the structural bias (Proposition~\ref{prop:adv_bias}) further compounds this imbalance: low-entropy groups, whose advantages are most distorted relative to the true reward landscape, nonetheless produce the strongest and most consistent gradient signals, while high-entropy groups—despite their advantages being closer to the reference frame—contribute weaker, noisier gradients.
As a result, the batch-level optimization is systematically dominated by the already-converging low-entropy tasks.
This creates a self-reinforcing cycle: as low-entropy groups improve, their accuracy rises and entropy drops further, widening the entropy gap and causing high-entropy groups to receive diminishing effective learning signal—ultimately leading to persistent inter-task imbalance.

This analysis motivates an entropy-conditioned intervention that rebalances the advantage signals across entropy regimes, rather than treating all normalized advantage distributions identically.

\subsection{Group Entropy-Controlled Policy Optimization}
\label{sec:gepo_method}

Consequently, we propose our group entropy-controlled policy optimization method.
For a prompt $x$ with group entropy $\mathcal{H}_{\text{g}}(x)$ and advantages $\{A_i\}_{i=1}^K$, we define a group entropy-based coefficient to shape the advantage as
\begin{equation}
    \hat{A}_i = \omega(A_i,\mathcal{H}_{\text{g}})A_i=
    \begin{cases}
        \alpha_{\text{low}}\cdot A_i  & \text{if } A_i > 0\ \text{and}\ \mathcal{H}_{\text{g}}(x) < \mathcal{H}_{\text{low}}^{(t)},\\
        \alpha_{\text{high}}\cdot A_i & \text{if } A_i < 0\ \text{and}\ \mathcal{H}_{\text{g}}(x) > \mathcal{H}_{\text{high}}^{(t)},\\
        A_i & \text{otherwise},
    \end{cases}
    \label{eq:adv_shaping}
\end{equation}
where $\alpha_{\text{high}}, \alpha_{\text{low}} \in (0, 1)$ are the scaling coefficients, and $\mathcal{H}_{\text{low}}^{(t)}$ and $\mathcal{H}_{\text{high}}^{(t)}$ denote the adaptive lower and upper entropy thresholds at training step $t$ (defined in the following paragraph).

Compared with global policy entropy, Equation \ref{eq:adv_shaping} offers several notable advantages. 
First, it is also task-aware: even within the same domain, prompts may vary considerably in difficulty, and group entropy provides a finer-grained characterization at the task level. 
Second, group entropy is naturally compatible with the grouped sampling structure of GRPO, allowing it to be incorporated without introducing additional sampling procedures or computational overhead beyond that already required by GRPO. 
Finally, by conditioning on group entropy, the advantage shaping attenuates the structural bias characterized by Propositions~\ref{prop:distortion}--\ref{prop:adv_bias}: it attenuates the over-exploitative signals from low-entropy groups that would otherwise drive premature convergence, and prevents the diluted negative signals in high-entropy groups from suppressing exploration before the policy has had a chance to discover correct reasoning paths, thereby promoting more balanced learning across entropy regimes.

\noindent\textbf{Asymmetric Advantage Shaping.}\hspace{0.3cm}
We impose the constraint $\alpha_{\text{high}} < \alpha_{\text{low}}$ based on the following empirical observation: in the low-entropy regime, the model is already confident and the policy distribution is sharply peaked.
In this regime, we empirically find that aggressively penalizing negative advantages (which would further increase entropy) may cause length collapse—a pathological behavior where the model dramatically shortens its responses to reduce per-token uncertainty.
By using a higher $\alpha_{\text{low}}$ for positive advantage attenuation in the low-entropy regime, we apply a gentler nudge that encourages exploration without triggering this instability. 
This asymmetry reflects the difference in the fragility of the two entropy regimes.

\noindent\textbf{Adaptive Entropy Thresholds.}\hspace{0.3cm}
Fixed entropy thresholds are particularly brittle in multi-task LLM post-training.
The appropriate entropy range depends on the task mixture, the initialization policy, and the training stage: different base models can exhibit substantially different entropy scales and controllability on the same tasks, and the entropy distribution further shifts as the policy improves. 
Therefore, we define entropy thresholds relative to the empirical group entropy distribution of the current batch. 
At step $t$, we compute the batch mean $\mu_H^{(t)}$ and standard deviation $\sigma_H^{(t)}$ of group entropy, and set
\begin{equation}
    \hat{\mathcal{H}}_{\text{low}}^{(t)} = \mu_H^{(t)} - \beta_{\text{low}} \sigma_H^{(t)}, \qquad
    \hat{\mathcal{H}}_{\text{high}}^{(t)} = \mu_H^{(t)} + \beta_{\text{high}} \sigma_H^{(t)}.
\end{equation}
Here, $\beta_{\text{low}}$ and $\beta_{\text{high}}$ control the width of the lower and upper entropy margins, respectively. 
We then apply exponential moving averages to obtain temporally stable thresholds:
\begin{equation}
    \mathcal{H}_{\text{low}}^{(t+1)} = (1-\gamma) \mathcal{H}_{\text{low}}^{(t)} + \gamma \hat{\mathcal{H}}_{\text{low}}^{(t)}, \qquad
    \mathcal{H}_{\text{high}}^{(t+1)} = (1-\gamma) \mathcal{H}_{\text{high}}^{(t)} + \gamma \hat{\mathcal{H}}_{\text{high}}^{(t)}.
\label{eq:adaptive_entropy_threshold}
\end{equation}
The smoothing coefficient $\gamma \in (0,1)$ governs the trade-off between responsiveness and stability.
This yields a distribution-aware and temporally smooth entropy band that adapts to model-specific entropy scales and changing exploration regimes.
\section{Experiments}
\subsection{Settings}\label{experimental_setting}
We select Intern-S1-mini and Qwen3.5-9B as our base models. Intern-S1-mini is a lightweight open-source multimodal reasoning model based on the same techniques as Intern-S1 \cite{bai2025interns1scientificmultimodalfoundation}, and Qwen3.5-9B is an open-source text-based reasoning model.
We use an in-house multimodal dataset in the training process. 
The dataset is designed to cover diverse task categories and knowledge domains, including mathematics, instruction following, engineering reasoning, physics, and chemistry.
Our implementation is based on LMDeploy \cite{2023lmdeploy} and Xtuner \cite{2023xtuner}.
During rollouts, we set the temperature to 1 and sample 16 responses per prompt. 
Training follows an off-policy RL setup with a batch size of 256 and a minibatch size of 32.
For hyperparameter settings, the scaling coefficients $\alpha_{\text{high}}, \alpha_{\text{low}}$ are set as 0.2 and 0.5.
The adaptive entropy thresholds coefficients $\beta_{\text{high}}$, $\beta_{\text{low}}, \ \text{and}\ \gamma$ are set as 0.3, 0.2 and 0.01.

For evaluation, we evaluate model performance on a diverse benchmark suite spanning mathematical reasoning, scientific reasoning, code generation, instruction following, and multimodal understanding, including GPQA \cite{rein2023gpqagraduatelevelgoogleproofqa}, LiveCodeBench \cite{jain2024livecodebenchholisticcontaminationfree}, IFBench \cite{pyatkin2025generalizingverifiableinstructionfollowing}, AIME2025, IMO-Bench-AnswerBench, CMPhysBench \cite{wang2025cmphysbenchbenchmarkevaluatinglarge}, MMMU-Pro \cite{yue2025mmmuprorobustmultidisciplinemultimodal}, MathVista \cite{lu2024mathvistaevaluatingmathematicalreasoning}, Physics \cite{feng2025physicsbenchmarkingfoundationmodels}, SFE \cite{zhou2025scientistsexamprobingcognitive}, MicroVQA \cite{burgess2025microvqamultimodalreasoningbenchmark}, and ChartQAPro \cite{masry2025chartqaprodiversechallengingbenchmark}. 
This broad evaluation suite enables us to assess both general reasoning ability and robustness across diverse domains, task formats, and multimodal settings.

We compare our approach against GRPO, AEPO \cite{shen2025qwenlong}, Clip-Cov, and KL-Cov \cite{cui2025entropy}.
For Clip-Cov, and KL-Cov, we use the same set of parameters as for Qwen2.5-7B from \cite{cui2025entropy}.
Besides, all other sampling and training parameters are consistent with the GEPO implementation.
To avoid the issues of token dropping, we clip the importance sampling weight as proposed in CISPO \cite{chen2025minimax}, which are also implemented in baseline methods.

\subsection{Main Results}
In this section, we present a comprehensive analysis of GEPO on Intern-S1-mini and Qwen3.5-9B.
All results are derived from the best checkpoint of our method and baselines.
We structure our analysis into three dimensions: the performance across multiple benchmarks, the training stability, and the training dynamics of entropy.

\begin{table}[!t]
  \centering
  \small
  \caption{Main results on thirteen benchmarks across baselines and our methods}
  \setlength{\tabcolsep}{2pt}
  \begin{tabular}{@{}lccccccc@{}}
    \toprule
    & \multicolumn{3}{c}{MATH} 
    & \multicolumn{2}{c}{Physics} 
    & \multirow{2}{*}{\makecell{Instruction\\Following}} 
    & \multirow{2}{*}{Code} \\
    \cmidrule(lr){2-4} \cmidrule(lr){5-6}
    & AIME25 & IMO-Bench & MathVista 
    & Physics & CMPhysBench \\
    \midrule
    \textit{Intern-S1-mini} & 74.6 & 42.3 & 70.5 & 37.4 & 29.2 & 29.8 & 39.4 \\
    + GRPO & 75.7 & 45.5 & 70.0 & 37.8 & 28.0 & 34.3 & 38.9 \\
    + AEPO & 79.5 & 47.5 & \textbf{71.5} & 36.1 & 28.2 & 32.2 & 38.3 \\
    + Clip-Cov & 76.7 & 49.8 & 69.8 & 32.4 & 30.6 & 31.2 & 38.3 \\
    + KL-Cov & 76.4 & 49.8 & 71.2 & 32.6 & \textbf{31.2} & 33.5 & \textbf{40.0} \\
    \rowcolor{gray!15}+ GEPO (ours) & \textbf{82.0} & \textbf{52.8} & 71.0 & \textbf{39.1} & 30.4 & \textbf{38.8} & 39.4 \\
    \midrule
    \textit{Qwen3.5-9B} & 82.5 & 58.5 & 85 & 51.0 & 35.8 & 62.2 & 57.7 \\
    + GRPO & 90.1 & 67.3 & \textbf{86.4} & 56.8 & 41.2 & 80.7 & \textbf{68.0} \\
    + AEPO & 84.2 & 63.8 & 85.3 & 53.2 & 41.5 & 68.0 & 60.6 \\
    + Clip-Cov & \textbf{91.9} & 69.3 & 85.5 & \textbf{59.0} & 42.9 & \textbf{81.3} & 64.0 \\
    + KL-Cov & 89.9 & 66.8 & 85.4 & 55.9 & 41.1 & 80.2 & 64.6 \\
    \rowcolor{gray!15}+ GEPO (ours) & 91.4 & \textbf{69.5} & 84.9 & 58.8 & \textbf{43.2} & 80.7 & 67.4 \\
    \midrule
    & \multicolumn{4}{c}{Knowledge \& Science} 
    & \multicolumn{2}{c}{Visual \& Chart} 
    & \multirow{2}{*}{\textbf{Avg.}} \\
    \cmidrule(lr){2-5} \cmidrule(lr){6-7}
    & MMLU-Pro & GPQA & MMMU-Pro & SFE 
    & ChartQAPro & MicroVQA \\
    \midrule
    \textit{Intern-S1-mini} & 75.0 & 65.1 & 55.0 & 45.0 & 47.5 & 50.4 & 50.9 \\
    + GRPO & 74.3 & 64.3 & 55.7 & 44.7 & 48.5 & 51.8 & 51.5 \\
    + AEPO & 74.7 & 62.5 & 57.4 & \textbf{46.5} & 46.7 & \textbf{52.5} & 51.8 \\
    + Clip-Cov & 75.0 & 65.7 & 55.8 & 43.9 & 48.6 & 51.5 & 51.5 \\
    + KL-Cov & 75.1 & 64.7 & 55.8 & 42.8 & \textbf{49.7} & 50.8 & 51.8 \\
    \rowcolor{gray!15}+ GEPO (ours) & \textbf{77.1} & \textbf{69.2} & \textbf{59.8} & 43.9 & 49.0 & 51.5 & \textbf{54.2} \\
    \midrule
    \textit{Qwen3.5-9B} & 81.6 & 81.9 & 70.1 & 53.1 & 63.4 & 62.9 & 65.0 \\
    + GRPO & 83.4 & \textbf{85.3} & 74.8 & 57.1 & 64.9 & 62.9 & 70.7 \\
    + AEPO & 82.5 & 81.6 & 70.9 & 55.2 & 64.0 & 61.4 & 67.1 \\
    + Clip-Cov & 83.3 & 85.0 & 76.3 & 57.7 & 62.9 & 63.1 & 70.9 \\
    + KL-Cov & \textbf{83.9} & 83.3 & 73.5 & 55.6 & \textbf{65.4} & 62.8 & 69.9 \\
    \rowcolor{gray!15}+ GEPO (ours) & 83.1 & 83.7 & \textbf{76.6} & \textbf{60.3} & 62.4 & \textbf{64.1} & \textbf{71.2} \\
    \bottomrule
  \end{tabular}
  \label{tab:entropy_adv_results}
\end{table}

\noindent\textbf{Performance Across Multiple Benchmarks.}\hspace{0.3cm}
As shown in Table~\ref{tab:entropy_adv_results}, existing methods exhibit notable learning imbalance across tasks.
AEPO achieves strong improvements on mathematics but yields limited gains or even underperformance on physics, instruction following, and scientific reasoning.
This suggests that its optimization pressure is disproportionately concentrated on certain task categories at the expense of others.
Similarly, GRPO shows moderate improvements on some tasks but limited or negative transfer on others, \eg, CMPhysBench and MMLU-Pro.
Clip-Cov and KL-Cov provide stronger coverage-aware regularization and improve several individual benchmarks, but their gains remain uneven across task categories.
On Intern-S1-mini, Clip-Cov improves IMO-Bench and CMPhysBench, while KL-Cov further improves code generation and ChartQAPro.
However, both methods still underperform GEPO in average performance, achieving 51.5 and 51.8 compared with GEPO's 54.2, and show clear regressions on tasks such as Physics and SFE.

In contrast, GEPO achieves the best performance on 7 out of 13 benchmarks on Intern-S1-mini, spanning mathematical reasoning, physics, instruction following, and scientific understanding.
For example, GEPO attains substantial gains on challenging reasoning benchmarks such as AIME25, IMO-Bench, and GPQA, while simultaneously improving instruction following and scientific reasoning.
GEPO maintains strong gains across most benchmarks and delivers the best overall average, demonstrating that the group entropy mechanism effectively balances optimization across heterogeneous tasks.

On Qwen3.5-9B, GEPO further demonstrates its generalization capability, achieving the highest average score of 71.2.
Clip-Cov is a particularly strong baseline on this model, obtaining the best results on AIME25, Physics, and instruction following, and reaching an average score of 70.9; KL-Cov leads on MMLU-Pro and ChartQAPro with an average score of 69.9.
Nevertheless, GEPO matches or surpasses these coverage-based baselines in overall performance and achieves the best results on IMO-Bench, CMPhysBench, MMMU-Pro, SFE, and MicroVQA.
In contrast, AEPO shows significant degradation on Qwen3.5-9B compared to GRPO, \eg, IFBench and LiveCodeBench, suggesting that its uniform entropy control is particularly fragile when transferred to a base model with different entropy characteristics.
GEPO's adaptive thresholds automatically calibrate to the new model's entropy scale without any model-specific tuning, validating the model-agnostic design of the adaptive entropy mechanism.

\begin{figure}
  \centering
  \includegraphics[width=1\textwidth]{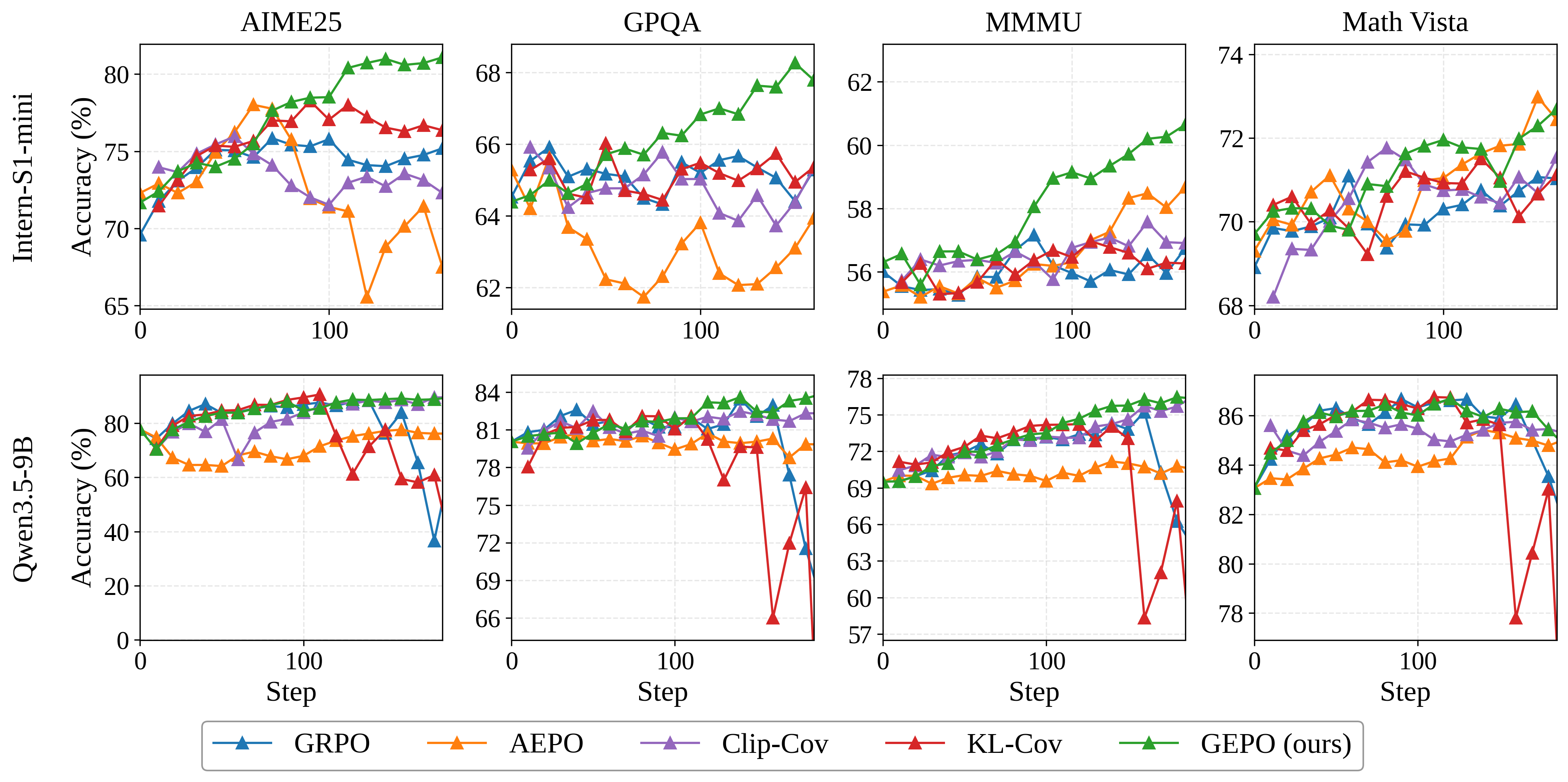}
  \caption{Typical validation curves during training process. For Intern-S1-mini and Qwen3.5-9B, GEPO improves smoothly throughout training, whereas baselines exhibit larger oscillations or late-stage performance collapse.}
  \label{fig:validation_curves}
\end{figure}

\noindent\textbf{Training Stability.}\hspace{0.3cm}
Beyond final performance, we examine the training dynamics through validation curves (Figure~\ref{fig:validation_curves}).
GEPO exhibits smooth and monotonically improving validation performance across both models and all benchmarks, with notably less variance compared to the baselines.

On Intern-S1-mini (top row), GEPO shows a clear upward trend on all four benchmarks with minimal fluctuation, whereas GRPO oscillates without clear improvement and AEPO suffers intermittent performance drops (\eg, AIME25 around step 80, GPQA declining throughout training).
The contrast is even more pronounced on Qwen3.5-9B (bottom row): GRPO undergoes catastrophic performance collapse around step 170, with AIME25 accuracy dropping from $\sim$85\% to $\sim$40\% and simultaneous degradation on GPQA, MMMU-Pro, and MathVista.
This collapse is symptomatic of unconstrained entropy dynamics under GRPO—without task-aware regulation, the policy entropy can grow unboundedly (as shown in Figure~\ref{fig:policy_entropy_curves}), eventually leading to degenerate outputs.
AEPO avoids complete collapse but exhibits persistent oscillations and generally underperforms GEPO, indicating that uniform entropy control is insufficient for heterogeneous task mixtures.
Clip-Cov and KL-Cov improve stability over unconstrained GRPO by introducing covaiance-aware constraints, but they still exhibit task-dependent fluctuations and plateau earlier on several validation benchmarks.
This suggests that controlling covaiance between action probability and the change in logits alone does not fully resolve the instability caused by heterogeneous task difficulty and task-specific exploration requirements.
Besides, KL-Cov shows training collapse on Qwen3,5-9B after 100 training steps, which is likely due to the fact that KL-Cov uses a fixed entropy threshold for all tasks, which may not be suitable for all tasks.
In contrast, GEPO maintains stable and steadily improving performance throughout training on both models, demonstrating robustness against the training instabilities that affect the baselines.
The stability of GEPO can be attributed to its group entropy control with adaptive thresholds: by independently regulating the exploration-exploitation balance for each prompt group, GEPO prevents the runaway entropy dynamics that trigger catastrophic collapse while maintaining sufficient exploration to sustain learning progress across all tasks.

\begin{figure}[!t]
  \centering
  \includegraphics[width=0.50\textwidth]{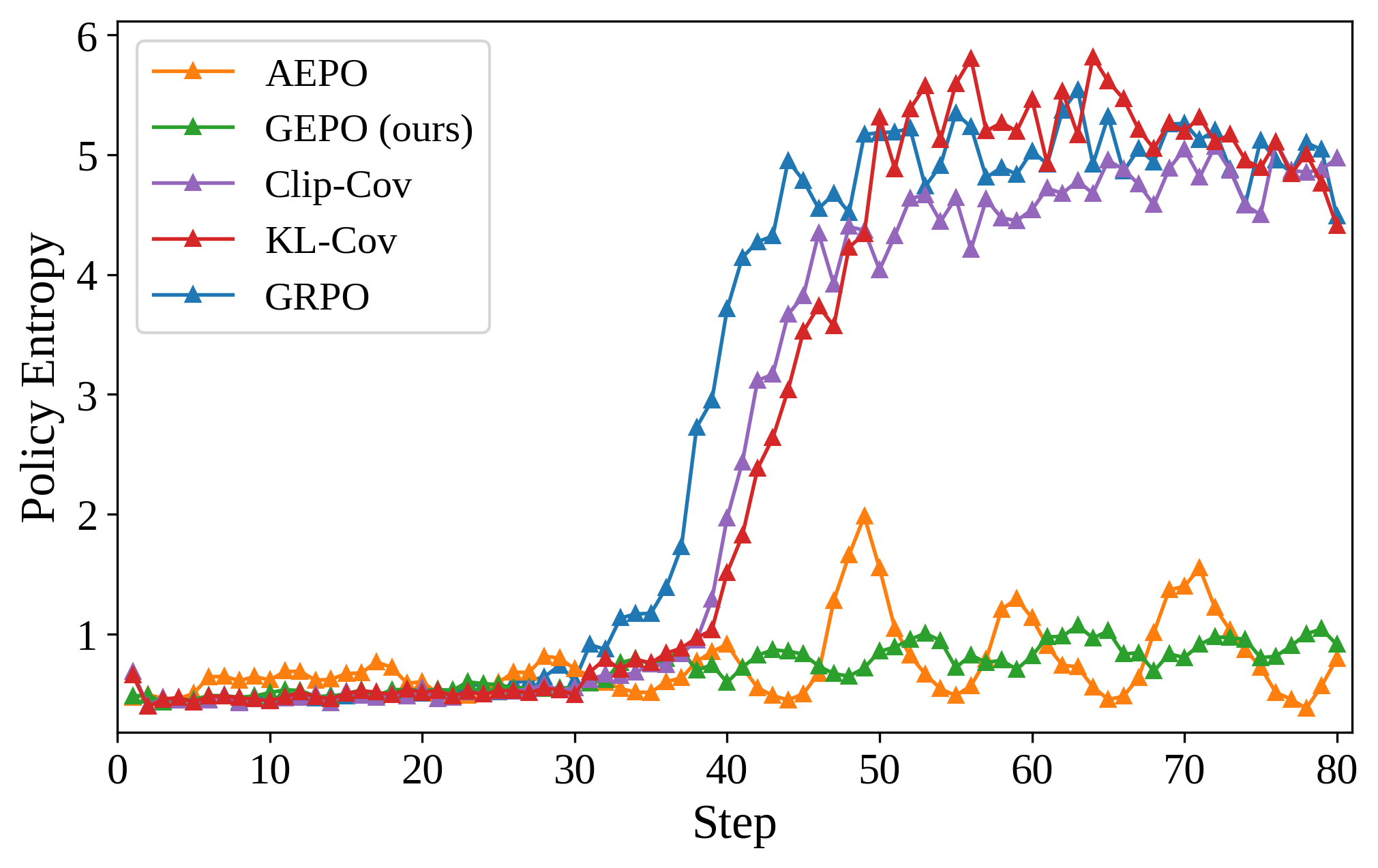}
  \caption{Policy entropy dynamics during training process of Intern-s1-mini. GEPO maintains a stable and bounded entropy trajectory, whereas the baselines exhibit runaway entropy growth or persistent oscillations.}
  \label{fig:policy_entropy_curves}
\end{figure}

\noindent\textbf{Training Dynamics of Entropy.}\hspace{0.3cm}
We further analyze how entropy evolves during training to understand the mechanism of GEPO.
As shown in Figure~\ref{fig:policy_entropy_curves}, GEPO maintains the healthy policy entropy throughout training, avoiding the premature entropy collapse that can lead to mode-seeking behavior and loss of generalization.
In contrast, GRPO suffers from runaway entropy growth, causing the model to produce degenerate and incoherent reasoning outputs, while Clip-Cov and KL-Cov show similar patterns of instability, with high entropy due to the lack of task-specific regulation.
AEPO, while avoiding complete collapse, exhibits substantial entropy oscillations throughout training, reflecting the instability of its uniform entropy control under heterogeneous task distributions. 

More importantly, the group entropy dynamics (Figure~\ref{fig:group_entropy_curves}) reveal a key distinction between GEPO and AEPO.
AEPO, which applies a uniform entropy control signal, drives the entropy of different tasks toward a similar convergence pattern, effectively imposing a one-size-fits-all exploration regime regardless of task characteristics.
GEPO, in contrast, preserves differentiated exploration levels across tasks: tasks requiring broader
exploration maintain relatively higher group entropy, while tasks with more deterministic solution patterns settle at lower entropy levels.
This heterogeneous entropy profile is consistent with the design of GEPO, which uses group entropy as a prompt-level diagnostic rather than forcing a uniform entropy target across all tasks.
This explains why GEPO achieves balanced performance improvements across diverse tasks (Dimension~1) with stable training dynamics (Dimension~2): the group entropy mechanism respects the inherent heterogeneity of multi-task learning and provides task-specific regulation without explicit task annotations.

\begin{figure}[!t]
  \centering
  \includegraphics[width=0.88\textwidth]{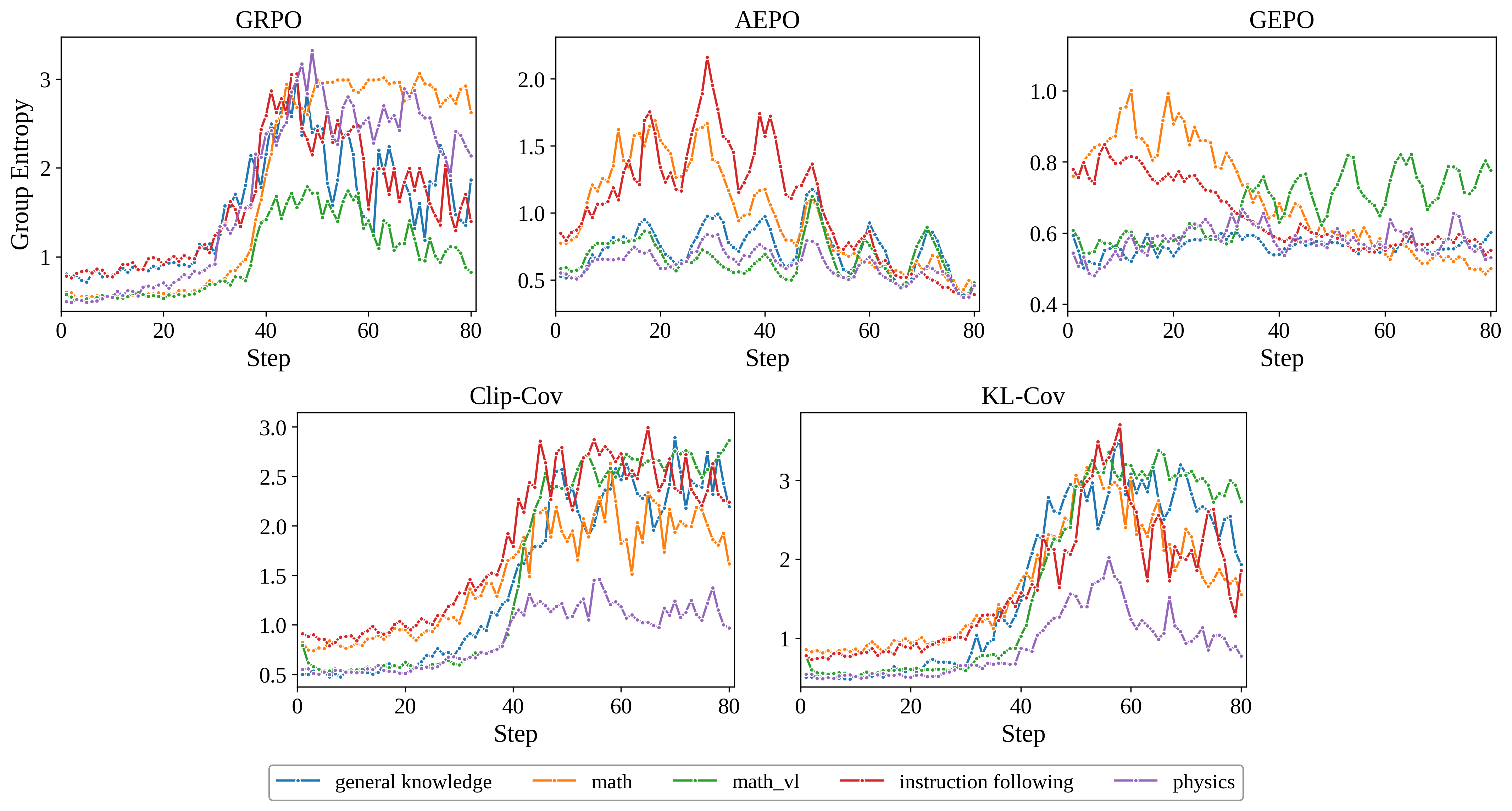}
  \caption{Task entropy dynamics during training process of Intern-s1-mini. GEPO maintains stable, domain-specific entropy regimes, while the baselines either destabilize or homogenize them.}
  \label{fig:group_entropy_curves}
\end{figure}

\noindent\textbf{Ablation Study.}\hspace{0.3cm}
To evaluate the effectiveness of asymmetric advantage shaping in GEPO, we conduct ablation studies using Intern-S1-mini.
As shown in Table~\ref{tab:ablation_asymmetric}, removing any component leads to performance degradation, confirming that all components are essential.
We observe that high entropy control contributes the most, as without it the policy suffers from unconstrained entropy growth and loses exploitation capability on hard reasoning tasks such as Physics  and AIME25.
Low entropy control prevents premature entropy collapse and maintains exploration diversity, particularly benefiting IMO-Bench and instruction following.
In addition, asymmetric advantage shaping with $\alpha_{\text{high}} < \alpha_{\text{low}}$ outperforms symmetrical Advantage Shaping, as applying a gentler nudge.

\begin{table}[!t]
  \centering
  \small
  \caption{Ablation results on asymmetric advantage shaping and bidirectional entropy control with Intern-S1-mini.}
  \setlength{\tabcolsep}{2pt}
  \begin{tabular}{@{}lccccccc@{}}
    \toprule
    & \multicolumn{3}{c}{MATH}
    & \multicolumn{2}{c}{Physics}
    & \multirow{2}{*}{\makecell{Instruction\\Following}}
    & \multirow{2}{*}{Code} \\
    \cmidrule(lr){2-4} \cmidrule(lr){5-6}
    & AIME25 & IMO-Bench & MathVista
    & Physics & CMPhysBench \\
    \midrule
    \rowcolor{gray!15} GEPO & \textbf{82.0} & \textbf{52.8} & 71.2 & \textbf{39.1} & 30.4 & \textbf{38.8} & \textbf{39.4} \\
    w/o low entropy control & 79.1 & 48.3 & \textbf{73.1} & 34.3 & \textbf{32.2} & 34.2 & 36.0 \\
    w/o high entropy control & 76.3 & 50.0 & 70.7 & 30.0 & 28.4 & 35.4 & 36.0 \\
    w/o asymmetric advantage shaping & 79.4 & 49.8 & 73.4 & 33.4 & 31.2 & 35.3 & 36.0 \\
    \midrule
    & \multicolumn{4}{c}{Knowledge \& Science}
    & \multicolumn{2}{c}{Visual \& Chart}
    & \multirow{2}{*}{\textbf{Avg.}} \\
    \cmidrule(lr){2-5} \cmidrule(lr){6-7}
    & MMLU-Pro & GPQA & MMMU-Pro & SFE
    & ChartQAPro & MicroVQA \\
    \midrule
    \rowcolor{gray!15} GEPO & \textbf{77.1} & \textbf{69.2} & 55.8 & 43.9 & 49.0 & \textbf{51.5} & \textbf{54.2} \\
    w/o low entropy control & 77.1 & 68.3 & 56.9 & \textbf{46.1} & 48.3 & 50.4 & 52.6 \\
    w/o high entropy control & 74.6 & 64.5 & 56.0 & 46.3 & 47.9 & 50.5 & 51.3 \\
    w/o asymmetric advantage shaping & 76.6 & 67.7 & \textbf{60.1} & 45.3 & \textbf{49.7} & 50.9 & 53.0 \\
    \bottomrule
  \end{tabular}
  \label{tab:ablation_asymmetric}
\end{table}

\section{Related Work}

Policy entropy has long been studied in reinforcement learning as a measure of action uncertainty and exploration, with entropy regularization and maximum-entropy objectives serving as fundamental mechanisms for encouraging exploration in stochastic policies \cite{haarnoja2018soft}.
More recently, entropy dynamics have attracted increasing attention in the post-training of LLMs, where researchers investigate their relationship with exploration, training stability, and reasoning performance \cite{cui2025entropy}.
Several studies have investigated the relationship between entropy dynamics and reasoning performance, suggesting that entropy evolution is closely associated with exploration behaviors, training stability, and the risk of premature convergence.
From an empirical perspective, \cite{cheng2026reasoning} characterize the role of high-entropy regions in facilitating reflective behaviors, pivotal reasoning steps, and exploration of alternative solution trajectories, while \cite{cui2025entropy} analyze entropy collapse during RL training and connect excessive entropy reduction with degraded reasoning performance.
Beyond empirical observations, \cite{wang2026entropy} provide a theoretical analysis of entropy evolution under policy optimization and its impact on training dynamics.
Motivated by these findings, recent approaches have introduced explicit entropy regulation mechanisms to improve RL training stability and exploration efficiency.
\cite{yang2025entropic} propose entropy control strategies to stabilize long-horizon RL optimization, while \cite{wan2026dsdr} preserve exploration diversity through global and local diversity regularization.
However, existing entropy-aware approaches primarily operate at the policy or token level, treating entropy as a global or local property of the policy distribution.
The impact of entropy variations across sampled groups, and their interaction with group-wise optimization mechanisms such as GRPO, remains largely unexplored.

\section{Conclusion}
We presented Group Entropy-Controlled Policy Optimization (GEPO) for robust multi-task reinforcement learning of large language models.
GEPO uses group entropy estimated from existing rollouts to perform asymmetric advantage shaping, moderating reinforcement in low-entropy groups and suppression in high-entropy groups through adaptive entropy boundaries.
This design regulates exploration-exploitation trade-off without explicit task annotations or additional sampling.
Across two base models and thirteen benchmarks, GEPO achieves the best average performance among the evaluated methods, exhibits stable optimization trajectories, and delivers balanced improvements across diverse task categories.
Our analysis further shows that preserving differentiated exploration regimes is more effective than driving heterogeneous tasks toward a uniform entropy pattern, establishing group entropy as a practical signal for scalable and stable multi-task RL post-training.

\clearpage
\bibliographystyle{plain}
\bibliography{ref}

@inproceedings{cheng2026reasoning,
  title={Reasoning with exploration: An entropy perspective},
  author={Cheng, Daixuan and Huang, Shaohan and Zhu, Xuekai and Dai, Bo and Zhao, Xin and Zhang, Zhenliang and Wei, Furu},
  booktitle={Proceedings of the AAAI Conference on Artificial Intelligence},
  volume={40},
  number={36},
  pages={30377--30385},
  year={2026}
}

@article{cui2025entropy,
  title={The entropy mechanism of reinforcement learning for reasoning language models},
  author={Cui, Ganqu and Zhang, Yuchen and Chen, Jiacheng and Yuan, Lifan and Wang, Zhi and Zuo, Yuxin and Li, Haozhan and Fan, Yuchen and Chen, Huayu and Chen, Weize and others},
  journal={arXiv preprint arXiv:2505.22617},
  year={2025}
}

@article{wang2026entropy,
  title={On the Entropy Dynamics in Reinforcement Fine-Tuning of Large Language Models},
  author={Wang, Shumin and Xie, Yuexiang and Zhang, Wenhao and Sun, Yuchang and Chen, Yanxi and Li, Yaliang and Zhang, Yanyong},
  journal={arXiv preprint arXiv:2602.03392},
  year={2026}
}

@article{yang2025entropic,
  title={EntroPIC: Towards Stable Long-Term Training of LLMs via Entropy Stabilization with Proportional-Integral Control},
  author={Yang, Kai and Xu, Xin and Chen, Yangkun and Liu, Weijie and Lyu, Jiafei and Lin, Zichuan and Ye, Deheng and Yang, Saiyong},
  journal={arXiv preprint arXiv:2511.15248},
  year={2025}
}

@article{wan2026dsdr,
  title={DSDR: Dual-Scale Diversity Regularization for Exploration in LLM Reasoning},
  author={Wan, Zhongwei and Shen, Yun and Dou, Zhihao and Zhou, Donghao and Zhang, Yu and Wang, Xin and Shen, Hui and Xiong, Jing and Tao, Chaofan and Zhong, Zixuan and others},
  journal={arXiv preprint arXiv:2602.19895},
  year={2026}
}

@article{team2025kimi,
  title={Kimi k1. 5: Scaling reinforcement learning with llms},
  author={Team, Kimi and Du, Angang and Gao, Bofei and Xing, Bowei and Jiang, Changjiu and Chen, Cheng and Li, Cheng and Xiao, Chenjun and Du, Chenzhuang and Liao, Chonghua and others},
  journal={arXiv preprint arXiv:2501.12599},
  year={2025}
}

@article{jaech2024openai,
  title={Openai o1 system card},
  author={Jaech, Aaron and Kalai, Adam and Lerer, Adam and Richardson, Adam and El-Kishky, Ahmed and Low, Aiden and Helyar, Alec and Madry, Aleksander and Beutel, Alex and Carney, Alex and others},
  journal={arXiv preprint arXiv:2412.16720},
  year={2024}
}

@article{hu2025open,
  title={Open-reasoner-zero: An open source approach to scaling up reinforcement learning on the base model},
  author={Hu, Jingcheng and Zhang, Yinmin and Han, Qi and Jiang, Daxin and Zhang, Xiangyu and Shum, Heung-Yeung},
  journal={arXiv preprint arXiv:2503.24290},
  year={2025}
}

@article{shao2024deepseekmath,
  title={Deepseekmath: Pushing the limits of mathematical reasoning in open language models},
  author={Shao, Zhihong and Wang, Peiyi and Zhu, Qihao and Xu, Runxin and Song, Junxiao and Bi, Xiao and Zhang, Haowei and Zhang, Mingchuan and Li, YK and Wu, Yang and others},
  journal={arXiv preprint arXiv:2402.03300},
  year={2024}
}

@misc{bai2025interns1scientificmultimodalfoundation,
      title={Intern-S1: A Scientific Multimodal Foundation Model},
      author={Lei Bai and Zhongrui Cai and Maosong Cao and Weihan Cao and Chiyu Chen and Haojiong Chen and Kai Chen and Pengcheng Chen and Ying Chen and Yongkang Chen and Yu Cheng and Yu Cheng and Pei Chu and Tao Chu and Erfei Cui and Ganqu Cui and Long Cui and Ziyun Cui and Nianchen Deng and Ning Ding and Nanqin Dong and Peijie Dong and Shihan Dou and Sinan Du and Haodong Duan and Caihua Fan and Ben Gao and Changjiang Gao and Jianfei Gao and Songyang Gao and Yang Gao and Zhangwei Gao and Jiaye Ge and Qiming Ge and Lixin Gu and Yuzhe Gu and Aijia Guo and Qipeng Guo and Xu Guo and Conghui He and Junjun He and Yili Hong and Siyuan Hou and Caiyu Hu and Hanglei Hu and Jucheng Hu and Ming Hu and Zhouqi Hua and Haian Huang and Junhao Huang and Xu Huang and Zixian Huang and Zhe Jiang and Lingkai Kong and Linyang Li and Peiji Li and Pengze Li and Shuaibin Li and Tianbin Li and Wei Li and Yuqiang Li and Dahua Lin and Junyao Lin and Tianyi Lin and Zhishan Lin and Hongwei Liu and Jiangning Liu and Jiyao Liu and Junnan Liu and Kai Liu and Kaiwen Liu and Kuikun Liu and Shichun Liu and Shudong Liu and Wei Liu and Xinyao Liu and Yuhong Liu and Zhan Liu and Yinquan Lu and Haijun Lv and Hongxia Lv and Huijie Lv and Qidang Lv and Ying Lv and Chengqi Lyu and Chenglong Ma and Jianpeng Ma and Ren Ma and Runmin Ma and Runyuan Ma and Xinzhu Ma and Yichuan Ma and Zihan Ma and Sixuan Mi and Junzhi Ning and Wenchang Ning and Xinle Pang and Jiahui Peng and Runyu Peng and Yu Qiao and Jiantao Qiu and Xiaoye Qu and Yuan Qu and Yuchen Ren and Fukai Shang and Wenqi Shao and Junhao Shen and Shuaike Shen and Chunfeng Song and Demin Song and Diping Song and Chenlin Su and Weijie Su and Weigao Sun and Yu Sun and Qian Tan and Cheng Tang and Huanze Tang and Kexian Tang and Shixiang Tang and Jian Tong and Aoran Wang and Bin Wang and Dong Wang and Lintao Wang and Rui Wang and Weiyun Wang and Wenhai Wang and Yi Wang and Ziyi Wang and Ling-I Wu and Wen Wu and Yue Wu and Zijian Wu and Linchen Xiao and Shuhao Xing and Chao Xu and Huihui Xu and Jun Xu and Ruiliang Xu and Wanghan Xu and GanLin Yang and Yuming Yang and Haochen Ye and Jin Ye and Shenglong Ye and Jia Yu and Jiashuo Yu and Jing Yu and Fei Yuan and Bo Zhang and Chao Zhang and Chen Zhang and Hongjie Zhang and Jin Zhang and Qiaosheng Zhang and Qiuyinzhe Zhang and Songyang Zhang and Taolin Zhang and Wenlong Zhang and Wenwei Zhang and Yechen Zhang and Ziyang Zhang and Haiteng Zhao and Qian Zhao and Xiangyu Zhao and Xiangyu Zhao and Bowen Zhou and Dongzhan Zhou and Peiheng Zhou and Yuhao Zhou and Yunhua Zhou and Dongsheng Zhu and Lin Zhu and Yicheng Zou},
      year={2025},
      eprint={2508.15763},
      archivePrefix={arXiv},
      primaryClass={cs.LG},
      url={https://arxiv.org/abs/2508.15763},
}

@misc{2023lmdeploy,
    title={LMDeploy: A Toolkit for Compressing, Deploying, and Serving LLM},
    author={LMDeploy Contributors},
    howpublished = {\url{https://github.com/InternLM/lmdeploy}},
    year={2023}
}

@misc{yue2025mmmuprorobustmultidisciplinemultimodal,
      title={MMMU-Pro: A More Robust Multi-discipline Multimodal Understanding Benchmark}, 
      author={Xiang Yue and Tianyu Zheng and Yuansheng Ni and Yubo Wang and Kai Zhang and Shengbang Tong and Yuxuan Sun and Botao Yu and Ge Zhang and Huan Sun and Yu Su and Wenhu Chen and Graham Neubig},
      year={2025},
      eprint={2409.02813},
      archivePrefix={arXiv},
      primaryClass={cs.CL},
      url={https://arxiv.org/abs/2409.02813}, 
}

@misc{rein2023gpqagraduatelevelgoogleproofqa,
      title={GPQA: A Graduate-Level Google-Proof Q\&A Benchmark}, 
      author={David Rein and Betty Li Hou and Asa Cooper Stickland and Jackson Petty and Richard Yuanzhe Pang and Julien Dirani and Julian Michael and Samuel R. Bowman},
      year={2023},
      eprint={2311.12022},
      archivePrefix={arXiv},
      primaryClass={cs.AI},
      url={https://arxiv.org/abs/2311.12022}, 
}

@misc{jain2024livecodebenchholisticcontaminationfree,
      title={LiveCodeBench: Holistic and Contamination Free Evaluation of Large Language Models for Code}, 
      author={Naman Jain and King Han and Alex Gu and Wen-Ding Li and Fanjia Yan and Tianjun Zhang and Sida Wang and Armando Solar-Lezama and Koushik Sen and Ion Stoica},
      year={2024},
      eprint={2403.07974},
      archivePrefix={arXiv},
      primaryClass={cs.SE},
      url={https://arxiv.org/abs/2403.07974}, 
}

@misc{pyatkin2025generalizingverifiableinstructionfollowing,
      title={Generalizing Verifiable Instruction Following}, 
      author={Valentina Pyatkin and Saumya Malik and Victoria Graf and Hamish Ivison and Shengyi Huang and Pradeep Dasigi and Nathan Lambert and Hannaneh Hajishirzi},
      year={2025},
      eprint={2507.02833},
      archivePrefix={arXiv},
      primaryClass={cs.CL},
      url={https://arxiv.org/abs/2507.02833}, 
}

@misc{wang2025cmphysbenchbenchmarkevaluatinglarge,
      title={CMPhysBench: A Benchmark for Evaluating Large Language Models in Condensed Matter Physics}, 
      author={Weida Wang and Dongchen Huang and Jiatong Li and Tengchao Yang and Ziyang Zheng and Di Zhang and Dong Han and Benteng Chen and Binzhao Luo and Zhiyu Liu and Kunling Liu and Zhiyuan Gao and Shiqi Geng and Wei Ma and Jiaming Su and Xin Li and Shuchen Pu and Yuhan Shui and Qianjia Cheng and Zhihao Dou and Dongfei Cui and Changyong He and Jin Zeng and Zeke Xie and Mao Su and Dongzhan Zhou and Yuqiang Li and Wanli Ouyang and Yunqi Cai and Xi Dai and Shufei Zhang and Lei Bai and Jinguang Cheng and Zhong Fang and Hongming Weng},
      year={2025},
      eprint={2508.18124},
      archivePrefix={arXiv},
      primaryClass={cs.LG},
      url={https://arxiv.org/abs/2508.18124}, 
}

@misc{lu2024mathvistaevaluatingmathematicalreasoning,
      title={MathVista: Evaluating Mathematical Reasoning of Foundation Models in Visual Contexts}, 
      author={Pan Lu and Hritik Bansal and Tony Xia and Jiacheng Liu and Chunyuan Li and Hannaneh Hajishirzi and Hao Cheng and Kai-Wei Chang and Michel Galley and Jianfeng Gao},
      year={2024},
      eprint={2310.02255},
      archivePrefix={arXiv},
      primaryClass={cs.CV},
      url={https://arxiv.org/abs/2310.02255}, 
}

@misc{feng2025physicsbenchmarkingfoundationmodels,
      title={PHYSICS: Benchmarking Foundation Models on University-Level Physics Problem Solving}, 
      author={Kaiyue Feng and Yilun Zhao and Yixin Liu and Tianyu Yang and Chen Zhao and John Sous and Arman Cohan},
      year={2025},
      eprint={2503.21821},
      archivePrefix={arXiv},
      primaryClass={cs.AI},
      url={https://arxiv.org/abs/2503.21821}, 
}

@misc{zhou2025scientistsexamprobingcognitive,
      title={Scientists' First Exam: Probing Cognitive Abilities of MLLM via Perception, Understanding, and Reasoning}, 
      author={Yuhao Zhou and Yiheng Wang and Xuming He and Ao Shen and Ruoyao Xiao and Zhiwei Li and Qiantai Feng and Zijie Guo and Yuejin Yang and Hao Wu and Wenxuan Huang and Jiaqi Wei and Dan Si and Xiuqi Yao and Jia Bu and Haiwen Huang and Manning Wang and Tianfan Fu and Shixiang Tang and Ben Fei and Dongzhan Zhou and Fenghua Ling and Yan Lu and Siqi Sun and Chenhui Li and Guanjie Zheng and Jiancheng Lv and Wenlong Zhang and Lei Bai},
      year={2025},
      eprint={2506.10521},
      archivePrefix={arXiv},
      primaryClass={cs.AI},
      url={https://arxiv.org/abs/2506.10521}, 
}

@misc{burgess2025microvqamultimodalreasoningbenchmark,
      title={MicroVQA: A Multimodal Reasoning Benchmark for Microscopy-Based Scientific Research}, 
      author={James Burgess and Jeffrey J Nirschl and Laura Bravo-Sánchez and Alejandro Lozano and Sanket Rajan Gupte and Jesus G. Galaz-Montoya and Yuhui Zhang and Yuchang Su and Disha Bhowmik and Zachary Coman and Sarina M. Hasan and Alexandra Johannesson and William D. Leineweber and Malvika G Nair and Ridhi Yarlagadda and Connor Zuraski and Wah Chiu and Sarah Cohen and Jan N. Hansen and Manuel D Leonetti and Chad Liu and Emma Lundberg and Serena Yeung-Levy},
      year={2025},
      eprint={2503.13399},
      archivePrefix={arXiv},
      primaryClass={cs.CV},
      url={https://arxiv.org/abs/2503.13399}, 
}

@misc{masry2025chartqaprodiversechallengingbenchmark,
      title={ChartQAPro: A More Diverse and Challenging Benchmark for Chart Question Answering}, 
      author={Ahmed Masry and Mohammed Saidul Islam and Mahir Ahmed and Aayush Bajaj and Firoz Kabir and Aaryaman Kartha and Md Tahmid Rahman Laskar and Mizanur Rahman and Shadikur Rahman and Mehrad Shahmohammadi and Megh Thakkar and Md Rizwan Parvez and Enamul Hoque and Shafiq Joty},
      year={2025},
      eprint={2504.05506},
      archivePrefix={arXiv},
      primaryClass={cs.CL},
      url={https://arxiv.org/abs/2504.05506}, 
}

@article{shen2025qwenlong,
  title={Qwenlong-l1. 5: Post-training recipe for long-context reasoning and memory management},
  author={Shen, Weizhou and Yang, Ziyi and Li, Chenliang and Lu, Zhiyuan and Peng, Miao and Sun, Huashan and Shi, Yingcheng and Liao, Shengyi and Lai, Shaopeng and Zhang, Bo and others},
  journal={arXiv preprint arXiv:2512.12967},
  year={2025}
}

@article{chen2025minimax,
  title={Minimax-m1: Scaling test-time compute efficiently with lightning attention},
  author={Chen, Aili and Li, Aonian and Gong, Bangwei and Jiang, Binyang and Fei, Bo and Yang, Bo and Shan, Boji and Yu, Changqing and Wang, Chao and Zhu, Cheng and others},
  journal={arXiv preprint arXiv:2506.13585},
  year={2025}
}

@book{tsybakov2009introduction,
  title={Introduction to Nonparametric Estimation},
  author={Tsybakov, Alexandre B.},
  year={2009},
  publisher={Springer},
  series={Springer Series in Statistics}
}

@inproceedings{haarnoja2018soft,
  title={Soft actor-critic: Off-policy maximum entropy deep reinforcement learning with a stochastic actor},
  author={Haarnoja, Tuomas and Zhou, Aurick and Abbeel, Pieter and Levine, Sergey},
  booktitle={International conference on machine learning},
  pages={1861--1870},
  year={2018},
  organization={Pmlr}
}

@book{sutton1998reinforcement,
  title={Reinforcement learning: An introduction},
  author={Sutton, Richard S and Barto, Andrew G and others},
  volume={1},
  number={1},
  year={1998},
  publisher={MIT press Cambridge}
}

@article{bellemare2016unifying,
  title={Unifying count-based exploration and intrinsic motivation},
  author={Bellemare, Marc and Srinivasan, Sriram and Ostrovski, Georg and Schaul, Tom and Saxton, David and Munos, Remi},
  journal={Advances in neural information processing systems},
  volume={29},
  year={2016}
}

@misc{2023xtuner,
    title={XTuner: A Toolkit for Efficiently Fine-tuning LLM},
    author={XTuner Contributors},
    howpublished = {\url{https://github.com/InternLM/xtuner}},
    year={2023}
}


\clearpage
\appendix
\section{Advantage Asymmetry under Per-Group Normalization}
\label{app:advantage_asymmetry_analysis}

We formalize how per-group normalization produces structurally different advantage distributions depending on the group accuracy, and quantify the resulting signal concentration effect described in Section~\ref{sec:group_entropy}.

\textbf{Setup.}
Consider a prompt group with $K$ responses and binary rewards $r_i \in \{0, 1\}$, where $n_+$ responses are correct ($r_i = 1$) and $n_- = K - n_+$ are incorrect ($r_i = 0$). (Since the normalized advantage $A_i = (r_i - \bar{r})/\sigma_r$ is invariant to affine transformations of the reward, this binary formulation applies without loss of generality to any two-valued reward scheme, \eg, $r_i \in \{-1, +1\}$.) The group accuracy is $p = n_+ / K$.
GRPO's per-group normalization yields advantages:
\begin{equation}
    A_i = \frac{r_i - \bar{r}}{\sigma_r}, \qquad \text{where } \bar{r} = p, \quad \sigma_r = \sqrt{p(1-p)}.
\end{equation}
Substituting, correct responses receive $A_+ = \frac{1-p}{\sqrt{p(1-p)}} = \sqrt{\frac{1-p}{p}}$, and incorrect responses receive $A_- = \frac{-p}{\sqrt{p(1-p)}} = -\sqrt{\frac{p}{1-p}}$.

\textbf{Signal concentration.}
The total positive signal equals $n_+ \cdot A_+ = K\sqrt{p(1-p)}$, which is symmetric and equals the magnitude of the total negative signal $n_- \cdot |A_-|$ by the zero-mean constraint.
However, the \emph{per-response} signal magnitudes are asymmetric:
\begin{equation}
    |A_+| = \sqrt{\frac{1-p}{p}}, \qquad |A_-| = \sqrt{\frac{p}{1-p}}.
\end{equation}
As $p\to0$, $|A_+|\to\infty$ while $|A_-|\to0$: the positive signal concentrates on a small number of correct responses, whereas the negative signal is distributed across many incorrect responses with small per-response magnitude. 
In our multi-task rollouts, such low-accuracy groups are empirically associated with higher group-entropy regimes.

\textbf{Skewness.}
The skewness of the advantage distribution in a single group with accuracy $p$ is:
\begin{equation}
    \text{skew}(A) = \frac{1 - 2p}{\sqrt{p(1-p)}}.
    \label{eq:advantage_skewness}
\end{equation}
This is zero when $p = 0.5$ (perfectly balanced), positive when $p < 0.5$ (more incorrect than correct), and grows without bound as $p \to 0$.
For our observed task accuracies: Physics ($p \approx 0.48$) yields skewness $\approx 0.08$, Math ($p \approx 0.39$) yields $\approx 0.45$, and Instruction Following ($p \approx 0.16$) yields $\approx 1.86$.

\textbf{Remark: pooled vs.\ single-group skewness.}
Equation~\ref{eq:advantage_skewness} gives the skewness for a \emph{single} group with accuracy exactly $p$.
The empirical skewness values reported in Observation~2 of the main text (Physics $\approx 0$, Math $\approx 0.75$, IF $\approx 2.32$) are instead computed by pooling advantages across \emph{all} groups of a given task.
Since different prompts within the same task have different accuracies $p_j$, the pooled skewness equals $\frac{1}{N}\sum_j \text{skew}(p_j)$ rather than $\text{skew}(\bar{p})$.
Because $\text{skew}(p) = (1-2p)/\sqrt{p(1-p)}$ is \emph{convex} on $(0, 0.5)$, Jensen's inequality gives
\[
    \frac{1}{N}\sum_{j=1}^{N} \text{skew}(p_j) \;\geq\; \text{skew}\!\left(\frac{1}{N}\sum_{j=1}^{N} p_j\right),
\]
with equality only when all groups share the same accuracy.
This accounts for the quantitative gap between the formula predictions and the empirical values, and implies that heterogeneous group accuracies \emph{amplify} the advantage asymmetry beyond what the single-group analysis predicts.


\section{Structural Bias from Entropy Heterogeneity}
\label{app:structural_bias}

This appendix provides the formal setup and complete proofs for Propositions~\ref{prop:distortion} and~\ref{prop:adv_bias} stated in the main text, as well as a finite-sample error decomposition that clarifies the distinct sources of estimation error in group-based advantage computation.

\subsection{Formal Setup}

For a prompt $x$, let $\mathcal{A}_x$ denote the finite response
space induced by the vocabulary and maximum generation length, with
$|\mathcal{A}_x|=N_x$. We define
\[
\pi_x(a):=\pi_\theta(a\mid x),
\qquad
\mu_x(a):=\mathbb{E}[R\mid x,a].
\]
During training, group responses are sampled independently from the
policy:
\[
a_1,\ldots,a_K\overset{\mathrm{i.i.d.}}{\sim}\pi_x.
\]

We introduce a reference measure $\nu_x$ over $\mathcal{A}_x$.
When explicitly stated, we take $\nu_x(a)=1/N_x$. The corresponding
reward CDFs are
\[
F_x^\nu(t):=\nu_x(\{a:\mu_x(a)\leq t\}),
\qquad
F_x^\pi(t):=\pi_x(\{a:\mu_x(a)\leq t\}).
\]

We define two reward cumulative distribution functions:
\begin{itemize}[leftmargin=*]
    \item \textbf{Reference reward CDF:} $F_x^\nu(t) := \nu_x(\{a:\mu_x(a)\le t\})$
    \item \textbf{Policy-weighted reward CDF:} $F_x^\pi(t) := \pi_x(\{a:\mu_x(a)\le t\})$
\end{itemize}
Standard group-based methods observe samples from $F_x^\pi$, not $F_x^\nu$. The discrepancy between the two is governed by how far the policy deviates from the reference measure.

\subsection{Proof of Proposition~\ref{prop:distortion} (Distribution Distortion Bound)}

\begin{theorem}[Restatement]
For any prompt $x$,
\[
\sup_t \left|F_x^\pi(t) - F_x^\nu(t)\right| \;\leq\; \mathrm{TV}(\pi_x, \nu_x).
\]
If $\nu_x$ is the uniform distribution over $\mathcal{A}_x$, then
\[
\mathrm{TV}(\pi_x, \nu_x) \;\leq\; \sqrt{\frac{1}{2}(\log N_x - H(x))}.
\]
\end{theorem}

\begin{proof}
For any threshold $t$, define the level set $B_t := \{a \in \mathcal{A}_x : \mu_x(a) \le t\}$. Then
\[
F_x^\pi(t) - F_x^\nu(t) = \pi_x(B_t) - \nu_x(B_t).
\]
Taking the supremum over all $t$ (equivalently, over all level sets $B_t$):
\[
\sup_t |F_x^\pi(t) - F_x^\nu(t)| \;\leq\; \sup_{B \subseteq \mathcal{A}_x} |\pi_x(B) - \nu_x(B)| \;=\; \mathrm{TV}(\pi_x, \nu_x).
\]
The first inequality holds because level sets $\{B_t\}$ form a subset of all measurable sets.

For the second part, when $\nu_x$ is uniform:
\[
\mathrm{KL}(\pi_x \| \nu_x) = \sum_a \pi_x(a) \log \frac{\pi_x(a)}{1/N_x} = \sum_a \pi_x(a) \log \pi_x(a) + \log N_x = \log N_x - H(x).
\]
Applying Pinsker's inequality (\cite{tsybakov2009introduction}) $\mathrm{TV}(P,Q) \le \sqrt{\frac{1}{2}\mathrm{KL}(P\|Q)}$:
\[
\mathrm{TV}(\pi_x, \nu_x) \;\leq\; \sqrt{\frac{1}{2}(\log N_x - H(x))}.
\]
\end{proof}

\subsection{Proof of Proposition~\ref{prop:adv_bias} (Advantage Bias Bound)}

Define the policy-weighted and reference moments:
\[
m_x^\pi := \mathbb{E}_{\pi_x}[\mu_x(a)], \quad (\sigma_x^\pi)^2 := \mathrm{Var}_{\pi_x}(\mu_x(a)), \quad m_x^\nu := \mathbb{E}_{\nu_x}[\mu_x(a)], \quad (\sigma_x^\nu)^2 := \mathrm{Var}_{\nu_x}(\mu_x(a)),
\]
and the corresponding standardized advantages:
\[
Z_x^\pi(a) = \frac{\mu_x(a) - m_x^\pi}{\sigma_x^\pi}, \qquad Z_x^\nu(a) = \frac{\mu_x(a) - m_x^\nu}{\sigma_x^\nu}.
\]

\begin{theorem}[Restatement]
Under the assumptions $|\mu_x(a)| \le M$ and $\sigma_x^\pi, \sigma_x^\nu \ge \sigma_0 > 0$, there exists $C = C(M, \sigma_0) > 0$ such that
\[
\sup_{a \in \mathcal{A}_x} |Z_x^\pi(a) - Z_x^\nu(a)| \;\leq\; C\,\mathrm{TV}(\pi_x, \nu_x).
\]
\end{theorem}

\begin{proof}
\textbf{Step 1: Mean difference.}
\[
|m_x^\pi - m_x^\nu| = \left|\sum_a (\pi_x(a) - \nu_x(a))\mu_x(a)\right| \le 2M \cdot \mathrm{TV}(\pi_x, \nu_x),
\]
where we used $\sum_a |p(a)-q(a)| = 2\,\mathrm{TV}(p,q)$ and $|\mu_x(a)| \le M$.

\textbf{Step 2: Second moment difference.}
Let $u_x^\pi := \mathbb{E}_{\pi_x}[\mu_x(a)^2]$ and $u_x^\nu := \mathbb{E}_{\nu_x}[\mu_x(a)^2]$. Similarly:
\[
|u_x^\pi - u_x^\nu| \le 2M^2 \cdot \mathrm{TV}(\pi_x, \nu_x).
\]

\textbf{Step 3: Variance difference.}
Since $(\sigma_x^\pi)^2 = u_x^\pi - (m_x^\pi)^2$ and $(\sigma_x^\nu)^2 = u_x^\nu - (m_x^\nu)^2$:
\begin{align*}
|(\sigma_x^\pi)^2 - (\sigma_x^\nu)^2| &\le |u_x^\pi - u_x^\nu| + |(m_x^\pi)^2 - (m_x^\nu)^2| \\
&\le 2M^2 \cdot \mathrm{TV} + |m_x^\pi + m_x^\nu| \cdot |m_x^\pi - m_x^\nu| \\
&\le 2M^2 \cdot \mathrm{TV} + 2M \cdot 2M \cdot \mathrm{TV} \\
&= 6M^2 \cdot \mathrm{TV}(\pi_x, \nu_x).
\end{align*}
Using $|a^2 - b^2| = |a-b||a+b| \ge |a-b| \cdot \sigma_0$ (since $\sigma_x^\pi + \sigma_x^\nu \ge \sigma_0$):
\[
|\sigma_x^\pi - \sigma_x^\nu| \le \frac{6M^2}{\sigma_0} \cdot \mathrm{TV}(\pi_x, \nu_x) =: C_1 \cdot \mathrm{TV}.
\]

\textbf{Step 4: Advantage difference.}
For any $a \in \mathcal{A}_x$:
\begin{align*}
|Z_x^\pi(a) - Z_x^\nu(a)| &= \left|\frac{\mu_x(a) - m_x^\pi}{\sigma_x^\pi} - \frac{\mu_x(a) - m_x^\nu}{\sigma_x^\nu}\right| \\
&\le \frac{|m_x^\pi - m_x^\nu|}{\sigma_x^\pi} + |\mu_x(a) - m_x^\nu| \cdot \left|\frac{1}{\sigma_x^\pi} - \frac{1}{\sigma_x^\nu}\right| \\
&\le \frac{2M \cdot \mathrm{TV}}{\sigma_0} + 2M \cdot \frac{|\sigma_x^\pi - \sigma_x^\nu|}{\sigma_0^2} \\
&\le \frac{2M}{\sigma_0} \cdot \mathrm{TV} + \frac{2M \cdot C_1}{\sigma_0^2} \cdot \mathrm{TV} \\
&= \underbrace{\left(\frac{2M}{\sigma_0} + \frac{12M^3}{\sigma_0^3}\right)}_{=:\,C} \cdot \mathrm{TV}(\pi_x, \nu_x).
\end{align*}
\end{proof}



\end{document}